\documentclass[reqno,a4paper,11pt]{amsart}
\usepackage{amsmath, amssymb, tikz, xcolor}
\usepackage{bm, bbm}
\usepackage[mathscr]{eucal}
 \usepackage[top=2.54cm, bottom=2.54cm, left=2.54cm, right=2.54cm]{geometry}

\usepackage[foot]{amsaddr} 
\allowdisplaybreaks
\usepackage{mathtools}
\usepackage{easyReview}

\usepackage[shortlabels]{enumitem}
\setlist{topsep=2pt, itemsep=2pt}

\usepackage{comment}

\usepackage[colorlinks=true,linkcolor=teal, citecolor=teal]{hyperref}

\usepackage[nameinlink,capitalize]{cleveref}
\crefname{subsection}{Subsection}{Subsections}

\crefname{equation}{}{}
\crefname{theo}{Theorem}{Theorems}
\crefname{coro}{Corollary}{Corollaries}
\crefname{prop}{Proposition}{Propositions}
\crefname{lemm}{Lemma}{Lemmas}
\crefname{exam}{Example}{Examples}
\crefname{assum}{Assumption}{Assumptions}

\newtheorem{theo}{Theorem}[section]

\newtheorem{prop}[theo]{Proposition}

\theoremstyle{definition}
\newtheorem{defi}[theo]{Definition}
\newtheorem{assum}[theo]{Assumption}

\newtheorem{rema}[theo]{Remark}

\numberwithin{equation}{section}

\newcommand{\bR}{\mathbb R}
\newcommand{\bN}{\mathbb N}

\newcommand{\bE}{\mathbb E}

\newcommand{\bP}{\mathbb P}
\newcommand{\bF}{\mathbb F}
\newcommand{\1}{\mathbbm{1}}

\newcommand{\bfL}{\mathbf L}

\newcommand{\bfh}{\mathbf h}

\newcommand{\cB}{\mathcal B}

\newcommand{\cF}{\mathcal F}

\newcommand{\cL}{\mathcal L}

\newcommand{\cX}{\mathcal X}

\newcommand{\rmF }{\mathrm{F}}

\newcommand{\tran}{\mathsf{T}}

\newcommand{\e}{\mathrm{e}}

\newcommand{\Leb}{\bm{\lambda}}

\newcommand{\pd}{\partial}

\newcommand{\od}{\mathrm{d}}

\newcommand{\nq}{\!\!}

\newcommand{\gamI}{\gamma}

\makeatletter
\newcommand*\bigcdot{\mathpalette\bigcdot@{.35}}
\newcommand*\bigcdot@[2]{\mathbin{\vcenter{\hbox{\scalebox{#2}{$\m@th#1\bullet$}}}}}
\makeatother

\newcommand{\trm}[1]{\textrm{#1}}

\def\b{\big}

\def\bb{\bigg}

\def\<{\left<}
\def\>{\right>}

\begin{document}

\title{On the grid-sampling limit SDE}

\author[Christian Bender]{Christian Bender$^{1}$}
\address{$^1$Department of Mathematics, Saarland University, Germany}
\email{bender@math.uni-saarland.de}

\author[Nguyen Tran Thuan]{Nguyen Tran Thuan$^{1,2}$}
\email{nguyen@math.uni-saarland.de}
\address{$^2$Department of Mathematics, Vinh University, 182 Le Duan, Vinh, Nghe An, Viet Nam}
\email{thuannt@vinhuni.edu.vn}

\thanks{}

\date{October 10, 2024}

\begin{abstract}
In our recent work \cite{BN24} we introduced the grid-sampling SDE as a proxy for modeling exploration in continuous-time reinforcement learning. In this note, we provide further motivation for the use of this SDE and discuss its wellposedness in the presence of jumps.

	\bigskip
	
	\noindent \textbf{Keyworks.} Exploration; Orthogonal martingale measures; Poisson random measures; Reinforcement learning, Stochastic differential equations. 
	
	\smallskip
	
\end{abstract}

\maketitle


\section{Introduction}

Initiated by works of Zhou and coauthors \cite{JZ22a,JZ22, JZ23, WZZ20, WZ20}, a theoretical framework and algorithms for reinforcement learning (RL) have been developed for controlled stochastic differential equations (SDEs) in continuous time, see, e.g., \cite{BN23, BN24,DDJ23, FJ22, FGLPS23,  GLZ24, GXZ22, HWZ23, RZ21, STZ23} for some of the recent literature in the field. The design of learning algorithms necessitates the modeling of responses of the SDE system to the execution of randomized controls. This problem led the authors of \cite{JZ22} to introduce the \emph{sample state process}, see also \cite{FGLPS23, GLZ24}. In \cref{sec:Fubini}, we elaborate on the presentation of the sample state process in \cite{FGLPS23} in a simplified setting of drift control with additive noise. We argue that, in this situation, the sampling formulation based on the theory of rich Fubini extensions \cite{Su06, SZ09} has no proper interpretation as response of the original SDE system to a randomized control. The key issue is that the sampling via an uncountable family of (essentially) pairwise independent random variables leads to an averaging effect by Sun's exact law of large numbers in  \cite{Su06}.  In order to avoid this averaging effect, we exploit an idea in \cite{STZ23} and restrict the sampling to a discrete-time grid which then leads to the \textit{grid-sampling SDE}. In \cref{sec:grid-sampling}, we provide sufficient conditions for the wellposedness of the grid-sampling SDE driven by a multidimensional Brownian motion and a Poisson random measure. We, then, show that the solution to this grid-sampling SDE also solves an SDE with deterministic coefficients driven by suitably chosen random measures. This random measure formulation of the grid-sampling SDE has been motivated in our recent work \cite{BN24}, and the result stated in \cref{sec:grid-sampling} makes the related results in \cite{BN24} more precise. In \cref{sec:limit}, we first recall the limit theorem of \cite{BN24}, which leads to the \textit{grid-sampling limit SDE}. This SDE models the limit dynamics of the grid-sampling SDE as the mesh-size of the sampling partition goes to zero. Finally, we prove the strong existence and uniqueness for this SDE under appropriate Lipschitz and integrability conditions on the coefficients. 

Throughout this note, we apply standard notation and refer to the ``notations'' section in \cite{BN24} for more information.

\section{The sample state process and rich Fubini extensions}\label{sec:Fubini}

In this section, we focus on the case of drift control with additive noise given in terms of a one-dimensional Brownian motion $B$. We discuss how to make sense of an SDE of the form
\begin{equation}\label{eq:sample_formally}
	Y_t=y_0+\int_0^t b(s,Y_s,\bfh(s,Y_s,\xi_s))\od s +\sigma B_t, \quad t \in [0, T],
\end{equation}
where, ideally, $(\xi_t)_{t\in [0,T]}$ is a family of independent random variables, which are uniformly distributed on the unit interval. Moreover, $\xi=(\xi_t)_{t\in [0,T]}$ is assumed to be independent of the Brownian motion $B$. 

Here, we assume that  $b\colon [0,T]\times \bR\times \bR\rightarrow \bR$ and  ${\bfh} \colon [0,T]\times \bR\times\bR\rightarrow \bR$ are measurable with respect to the standard Borel $\sigma$-fields, and $y_0\in \bR$. In  \eqref{eq:sample_formally}, we think of $\bfh$ as the \textit{randomized feedback control} associated with a \textit{(feedback) relaxed} control $h$ defined on $[0, T] \times \bR $ with values in the family of probability distributions on  $\bR$. Based on the observation $Y_s$, the actor first chooses the probability distribution $h(s,Y_s)$ depending on $(s,Y_s)$. Now, if $\bfh(s,y,u)$ denotes, e.g., the quantile function of the distribution $h(s,y)$,  then $\bfh(s,Y_s,\xi_s)$ is a sample drawn from   $h(s,Y_s)$,  and the drawing mechanism based on $\xi_s$ is independent of the past information generated by $(B_t,\xi_t)_{0\leq t< s}$. The key technical issue is that the process  $\xi \colon [0, T] \times \Omega \to \bR$ cannot be $\cB([0,T]) \otimes \cF/\cB(\bR)$-measurable, see, e.g., Proposition 2.1 in \cite{Su06}, and so the integral in \eqref{eq:sample_formally} is not well-defined.

To overcome this problem, several authors, e.g., \cite{JZ23, FGLPS23}, have pointed to  the framework of rich Fubini extensions by \cite{Su06} for defining the sample state process as the solution of a suitable reformulation of \eqref{eq:sample_formally}, even beyond the case of drift control. Our discussion below elaborates the one in \cite{FGLPS23}. For simplicity, we let $T = 1$ but still write $[0, T]$ and $\int_0^T$ instead of $[0, 1]$ and $\int_0^1$, respectively, to distinguish the time- and space-variables.

According to \cite[Theorem 1]{SZ09}, there exist an extension $([0,T], \Lambda,\rho)$ of the Lebesgue probability space $([0,T],\cB([0,T]),\Leb_{[0, T]})$ and some probability space $(\Omega_1,\cF_1,\bP_1)$ such that the product space $([0,T] \times \Omega_1  , \Lambda \otimes \cF_1, \rho \otimes \bP_1)$ has   a rich Fubini extension  $([0,T] \times \Omega_1  , \Lambda \boxtimes \cF_1, \rho \boxtimes \bP_1)$, i.e., the following properties hold:
\begin{enumerate}
	\item There is a $\Lambda \boxtimes \cF_1/\mathcal{B}(\bR)$-measurable process $\xi \colon [0,T] \times \Omega_1 \rightarrow \bR$ such that,  for $\rho$-a.e. $t\in [0,T]$, $\xi_t$ is uniformly distributed on $[0,1]$ and independent of $\xi_s$ for   $\rho$-a.e. $s\in [0,T]$.
	\item For any $\rho \boxtimes \bP_1$-integrable function $F$, iterated integration is meaningful and
	\begin{align*}
&	\int_{[0, T] \times \Omega_1} F(t, \omega_1) (\rho\boxtimes \bP_1)(\od t, \od \omega_1) \\
& \quad = \int_{\Omega_1}  \bb(\int_0^T F(t, \omega_1)\rho(\od t)\bb) \bP_1 (\od \omega_1)= \int_0^T \bb(\nq \int_{\Omega_1} F(t, \omega_1) \bP_1(\od \omega_1)\bb)\rho(\od t),
		\end{align*}
\end{enumerate}
see \cite[Definition 3]{SZ09} or \cite[Definition 2.2]{Su06} for the complete statement of (2).

Moreover, we denote by $(\Omega_2,\cF_2,\bP_2)$ a probability space, which carries a one-dimensional Brownian motion $B$ with respect to its own filtration. We consider the usual product space $(\Omega,\cF,\bP) :=(\Omega_1\times \Omega_2,\cF_1\otimes \cF_2,\bP_1\otimes \bP_2)$
and extend $\xi$ and $B$ to mappings on $[0,T]\times \Omega_1\times\Omega_2$ by setting $\xi_t(\omega_1,\omega_2)=\xi_t(\omega_1)$ and $B_t(\omega_1,\omega_2)=B_t(\omega_2)$, respectively. Then, there is a $\rho$-null set $N_\rho \in \Lambda$  such that for every $t\in [0,T]\backslash N_\rho$, $\xi_t$ is a random variable (i.e., $\cF/\cB(\bR)$-measurable) and, by the product construction, the families $(\xi_t)_{t\in [0,T]\backslash N_\rho}$ and $(B_t)_{t\in [0,T]}$ are independent. Note that $\xi$ ``almost'' satisfies the properties required for the ideal sampling procedure mentioned above. However,   $\xi$ is not  $\mathcal{B}([0,T])\otimes \cF$-measurable and, hence, not predictable in the usual sense. Instead, $\xi$ only satisfies the weaker measurability property with respect to the larger $\sigma$-field $(\Lambda \boxtimes \cF_1)\otimes\cF_2$. As stochastic integration of stochastic processes with this measurability property is beyond the scope of the classical It\^o calculus, we only consider the drift control case. The following SDE is then a proper reformulation of \eqref{eq:sample_formally} in the framework of rich Fubini extensions:
\begin{equation}\label{eq:SDE_Fubini}
	Y_t=y_0+\int_0^t b(s,Y_s,\bfh(s,Y_s,\xi_s))\rho(\od s) +\sigma B_t, \quad t \in [0, T].
\end{equation}
We first motivate the notion of a solution to this equation.
Since $\rho(\{s\})=\Leb_{[0,T]}(\{s\})=0$ for every $s\in [0,T]$, integrals with respect to $\rho$ are continuous as functions in the upper integration limit. Therefore, a  solution $Y$ to \eqref{eq:SDE_Fubini} should have continuous paths and, then, should be $\mathcal{B}([0,T])\otimes \cF$-measurable. Note that for a process $Y$ with these measurability properties the maps
$$
(s,\omega) \mapsto  b(s,Y_s,\bfh(s,Y_s,\xi_s)), \quad (s,u,\omega) \mapsto  b(s,Y_s,\bfh(s,Y_s,u))
$$
are $(\Lambda \boxtimes \cF_1)\otimes\cF_2/\mathcal{B}(\bR)$ and $\mathcal{B}([0,T])\otimes \mathcal{B}(\bR) \otimes \cF/\mathcal{B}(\bR)$-measurable, respectively. By the definition of the Fubini extension, there is a ${\bP}$-null set $\mathcal{N}_0$ such that for every $\omega\in \Omega\backslash \mathcal{N}_0$, the map $s\mapsto  b(s,Y_s(\omega),\bfh(s,Y_s(\omega),\xi_s(\omega)))$ is $\Lambda/\cB(\bF)$-measurable, and so the integral in \eqref{eq:SDE_Fubini}  ``makes sense''.
\begin{defi}
	We say, a map $Y \colon [0,T]\times \Omega \rightarrow \bR$ is a \emph{solution} to \eqref{eq:SDE_Fubini}, if
	\begin{itemize}
		\item $Y$ is $\mathcal{B}([0,T])\otimes \cF$-measurable and has continuous paths;
		\item There is a ${\bP}$-null set $\mathcal{N}_0$ such that for every $\omega \in \Omega\backslash \mathcal{N}_0$, the map $$[0,T]\ni s \mapsto b(s,Y_s(\omega),\bfh(s,Y_s(\omega),\xi_s(\omega)))$$ is $\rho$-integrable and equation \eqref{eq:SDE_Fubini} is satisfied for every $(t,\omega)\in [0,T]\times (\Omega \backslash \mathcal{N}_0)$.
	\end{itemize}
\end{defi}

Since the function $t\mapsto \xi_t(\omega)$ is $\Lambda/\mathcal{B}(\bR)$-measurable for $\bP$-almost every $\omega\in \Omega$ by the definition of the Fubini extension, we can introduce the random measure
$$
\hat\rho(\omega;\od t, \od u) = \delta_{\xi_{t}(\omega)}(\od u) \rho(\od t)
$$
on $\Lambda\otimes \mathcal{B}(\bR)$.
If $Y$ is a solution to \eqref{eq:SDE_Fubini}, then, by (classical) Fubini's theorem, $Y$ satisfies
\begin{equation}\label{eq:SDE_Fubini_2}
	Y_t=y_0+\int_{(0,t]\times[0,1]} b(s,Y_s,\bfh(s,Y_s,u))\hat\rho(\od s,\od u) +\sigma B_t, \quad t \in [0, T]
\end{equation}
outside a $\bP$-null set.

The next theorem shows that the restriction of the random measure $\hat \rho$ to $\mathcal{B}([0,T])\otimes \mathcal{B}(\bR)$ is nothing but the Lebesgue measure  on $[0,T]\times[0,1]$.
\begin{theo}\label{lemm:coincide-measures}
	There is a ${\bP}$-null set $\mathcal{N}$ such that for every $\omega\in \Omega\backslash \mathcal{N}$ and $A\in \mathcal{B}([0,T])\otimes \mathcal{B}(\bR)$,
	$$
	\hat\rho(\omega;A)=(\Leb_{[0,T]}\otimes \Leb_{[0,1]})(A),
	$$
	where in slight abuse of notation we write $\Leb_{[0,1]}(B)=\Leb_{\bR}(B\cap [0,1])$ for $B\in \cB(\bR)$.
\end{theo}
\begin{proof}
	Let $A=B\times C \in \mathcal{B}([0,T])\otimes \mathcal{B}(\bR)$. If  $\Leb_{[0,T]}(B)>0$, then, by Sun's exact law of large numbers \cite[Theorem 2.6]{Su06}, for $\bP$-almost every $\omega \in \Omega$,
	$$
	\hat\rho(\omega;A)= \int_B \1_{\{\xi_t(\omega)\in C\}}\rho(\od t)=\int_B \bE[\1_{\{\xi_t\in C\}}]\rho(\od t)=\Leb_{[0,1]}(C)\rho(B)=(\Leb_{[0,T]}\otimes \Leb_{[0,1]})(A).
	$$
	If  $\Leb_{[0,T]}(B)=0$, then obviously both sides of the previous equation are zero.
	Thus, we find a  ${\bP}$-null set $\mathcal{N}$ such that for every $\omega\in \Omega \backslash \mathcal{N}$, the measures $\hat\rho(\omega;\cdot)$ and $\Leb_{[0,T]}\otimes \Leb_{[0,1]}$ coincide on all Cartesian products of subintervals  with rational endpoints. Now, Dynkin's $\pi$-$\lambda$ theorem completes the proof.
\end{proof}
Since $(s,u) \mapsto  b(s,Y_s{(\omega)},\bfh(s,Y_s{(\omega)},u))$ is  $\mathcal{B}([0,T])\otimes \mathcal{B}(\bR)$-measurable for any $\omega\in \Omega$, \cref{lemm:coincide-measures} and \eqref{eq:SDE_Fubini_2} imply that every solution $Y$ to \eqref{eq:SDE_Fubini} solves
\begin{equation}\label{eq:SDE_Fubini_3}
	Y_t=y_0+\int_{0}^t \! \int_0^1 b(s,Y_s,\bfh(s,Y_s,u)) \od u \od s +\sigma B_t, \quad t \in [0, T]
\end{equation}
outside a $\bP$-null set. Note that  \eqref{eq:SDE_Fubini_3} coincides with the \textit{exploratory SDE} introduced in \cite{WZZ20}.
This argument shows that the sample SDE \eqref{eq:SDE_Fubini} and the exploratory SDE \eqref{eq:SDE_Fubini_3} are equivalent in the pathwise sense (and not just in law). In particular, if the SDE \eqref{eq:SDE_Fubini} is solvable then we can choose a version of the solution which does not depend on $\xi$. Therefore, a solution $Y$ of \eqref{eq:SDE_Fubini} has no proper meaning as response  to a sample drawn from a measure-valued control. This point may become even more transparent, if we consider the special case $b(t,y,z)=z$, $\bfh(t,y,u)=\Phi^{-1}(u)$ (where $\Phi$ is the cumulative distribution function of a standard Gaussian), i.e., the measure-valued control is (independent of time and state) a standard Gaussian distribution. Then, \eqref{eq:SDE_Fubini} takes the form
$$
Y_t=y_0+\int_0^t \Phi^{-1}(\xi_s) \rho(\od s) +\sigma B_t.
$$
Hence, applying Sun's law of large numbers again and noting that $\bE[\Phi^{-1}(\xi_s)]=0$ for  $\rho$-a.e. $s \in [0, T]$, we obtain
$$
Y_t=y_0+\sigma B_t,
$$
and any dependence on the ``randomization'' process $\xi$ has disappeared.

The bottom line is that the integral in \eqref{eq:SDE_Fubini} is just a technically and notationally {different} way to re-write integration with respect the Lebesgue measure on $[0,T]\times [0,1]$ and does not model the execution of controls with values in the set of probability measures.

\section{The grid-sampling SDE}\label{sec:grid-sampling}

The problems discussed in the previous section illustrate that control randomization based on  (essentially) pairwise independent families $\xi=(\xi_t)_{t\in [0,T]}$ of uniform random variables may not be suitable in continuous time. Motivated by \cite{STZ23}, we replaced $\xi$ by a piecewise constant interpolation of finitely many independent uniform random variables on a finite grid $\Pi$ of $[0,T]$ as in \cite{BN24}. In this section, we complement the discussion of the resulting grid-sampling SDE in \cite{BN24} by rigorous wellposedness results. 

Let $(\Omega,\cF,\bar\bF,\bP)$ be a filtered probability satisfying the usual conditions. We assume that this probability space carries a $p$-dimensional Brownian motion $B$ and an (inhomogeneous) Poisson random measure $N$ (independent of $B$) with intensity $\nu(\od t, \od z) = \nu_t(\od z) \od t$ on $[0,T]\times\bR^q_0$. Here, $(\nu_t(\od z))_{t \in [0, T]}$ is a transition kernel consisting of L\'evy measures, i.e., $\nu_t$ is a Borel measure on $\bR^q_0: = \bR^q \backslash\{0\}$ with $\int_{\bR^q_0} (|z|^2 \wedge 1) \nu_t(\od z) < \infty$, $t \in [0, T]$. For simplicity, we impose the square-integrability condition
\begin{align}\label{assumption:Levy-measure-r}
	\int_0^T \nq \int_{\bR^q_0} |z|^2 \nu_t(\od z)  \od t < \infty,
\end{align}
which corresponds to the case $\mathfrak{r}=\infty$ in \cite{BN24}. Denote by $\tilde N$  the compensated measure of $N$.

 Let $\Pi$ be a partition of $[0,T]$ with grid points $0=t_0<t_1<\cdots<t_n=T$ for some $n\in \bN$. We denote the mesh-size of $\Pi$ by $|\Pi|:=\max_{1 \le i \le n}|t_i-t_{i-1}|$ and suppose that the probability space carries an independent family $(\xi_1,\ldots, \xi_n)$ of uniforms on $[0,1]^d$ independent of $(B,N)$. For the control randomization on the grid, we define
the \textit{grid-sampling process} $\xi^\Pi = (\xi^\Pi_t)_{t \in [0, T]}$ by
\begin{align*}
	\xi^\Pi_t :=\sum_{j=1}^n \xi_j {\bf 1}_{(t_{j-1},t_j]}(t),\quad t\in [0,T].
\end{align*}
Denote by $\bF^\Pi = (\cF^\Pi_t)_{t \in [0, T]}$  the right-continuous, augmented version of the filtration generated by $(B,N,\xi^\Pi)$. Then, the process $\xi^\Pi$ is left-continuous and adapted and, thus, is $\bF^\Pi$-predictable.

In the general setting of this section, we assume that controls take values in $\bR^d$. Recall that, for every probability distribution $P$ on $\cB(\bR^d)$ there is a Borel-measurable function $H\colon [0,1]^d\rightarrow \bR^d$ such that the random variable $H(\eta)$ has distribution $P$ for any uniform random variable $\eta$ on the unit cube $[0,1]^d$ . Hence, the following $\bR^m$-valued \textit{grid-sampling SDE} for a \textit{$\xi^\Pi$-randomized (feedback) policy} ${\bfh}\colon [0, T] \times \bR^m \times [0, 1]^d \to \bR^d$ can be considered as the execution of a probability measure-valued feedback control in a setting with controlled drift, diffusion, and jumps, where the randomization is performed by the grid-sampling process $\xi^\Pi$:
 \begin{align}\label{eq:SDE-Pi}
	\od  X^{\Pi,\bfh}_t & = b(t,  X^{\Pi,\bfh}_{t-}, {\bfh}(t, X^{\Pi,\bfh}_{t-},\xi^{\Pi}_t )) \od t + a(t,  X^{\Pi,\bfh}_{t-},  {\bfh}(t, X^{\Pi,\bfh}_{t-},\xi^{\Pi}_t )) \od B_t \notag \\
	& \quad + \int_{\bR^q_0} \gamI(t,  X^{\Pi,\bfh}_{t-},  {\bfh}(t, X^{\Pi,\bfh}_{t-},\xi^{\Pi}_t ), z) \tilde N(\od t, \od z),
\end{align}
with initial condition $X^{\Pi,\bfh}_0=x\in \bR^m$ and measurable coefficients $b \colon [0, T] \times \bR^m \times \bR^d \to \bR^m$, $a \colon [0, T] \times \bR^m \times \bR^d \to \bR^{m \times p}$, $\gamma \colon [0, T] \times \bR^m \times \bR^d \times \bR^q_0 \to \bR^m$.

For $\bfh$ as above, we define the functions $b_\bfh \colon [0,T]\times \bR^m\times [0,1]^d \rightarrow \bR^m$, $a_\bfh \colon [0,T]\times \bR^m\times [0,1]^d \rightarrow \bR^{m\times p}$, and $\gamma_\bfh \colon [0,T]\times \bR^m\times [0,1]^d\times \bR^q_0 \rightarrow \bR^{m}$
via
\begin{align*}
	& b_\bfh(s, x, u): = b(s, x, \bfh(s, x, u)), \quad 	a_\bfh(s, x, u): = a(s, x, \bfh(s, x, u)),\notag\\
	&	\gamma_\bfh(s, x, u, z): = \gamma(s, x, \bfh(s, x, u), z).
\end{align*}
 Assuming that these functions are Borel-measurable, the predictability of $\xi^\Pi$ implies that the random fields  $(b_\bfh(s, x, \xi^{\Pi}_s))$, $(a_\bfh(s, x, \xi^{\Pi}_s))$, and $(\gamma_\bfh(s, x, \xi^{\Pi}_s,z))$ indexed by $(s, x) \in [0, T] \times \bR^m$ and $(s, x,z) \in [0, T] \times \bR^m\times \bR^q_0$, respectively, are $\bF^\Pi$-predictable as well. Hence, all integrals in \eqref{eq:SDE-Pi} can be considered in the classical sense and \eqref{eq:SDE-Pi} has a sound interpretation as response to a randomized control. 
 
 \begin{assum}\label{ass:Lip_grid}
 	The coefficient functions $b_\bfh$, $a_\bfh$, $\gamma_\bfh$ are Borel-measurable and satisfy the following integrability condition: The function
 	\begin{align*}
 		G_0(s) := \int_{[0, 1]^d} \bb[ |b_\bfh(s, 0, u)|^2  +  \|a_\bfh(s, 0, u)\|_{\rmF}^2 +  \int_{\bR^q_0} |\gamma_\bfh(s, 0, u, z)|^2 \nu_s(\od z)\bb] \od u
 	\end{align*}
 	takes finite values for all $s \in [0, T]$ and $G_0 \in \bfL^1([0, T], \Leb_{[0, T]})$, where 
 		\begin{align*}
 			\|a_\bfh\|_{\rmF} : = \sqrt{\mathrm{trace}[a_\bfh^\tran a_\bfh]}
 		\end{align*}
 	denotes the Frobenius norm.
 	 Moreover, $b_\bfh$, $a_\bfh$, $\gamma_\bfh$ are Lipschitz continuous in the space variable $x$ in the following sense: There is a constant $K_{\textnormal{Lip}} \ge 0$ independent of $s$ and $u$ such that the following conditions hold for any $s \in [0, T]$, $u\in [0,1]^d$, and $x_1, x_2 \in \bR^m$:  
 	\begin{align*}
 		& |b_\bfh(s, x_1, u) - b_\bfh(s, x_2, u)|^2+ \|a_\bfh(s, x_1, u) - a_\bfh(s, x_2,u)\|_{\rmF}^2 \\
 		 &\quad +\int_{\bR^q_0} |\gamma_\bfh(s, x_1, u, z) - \gamma_\bfh(s, x_2, u, z)|^2\nu_s(\od z)   \leq K_{\textnormal{Lip}}|x_1 - x_2|^2.
 	\end{align*}
 \end{assum}

 \begin{theo}
 	Under \cref{ass:Lip_grid}, the grid-sampling SDE for policy ${\bfh}$ has a unique (up to an indistinguishability)  strong solution $X^{\Pi,\bfh}$, for any choice of the partition $\Pi$.  Moreover, there is a constant independent of the initial condition $x$ and the grid $\Pi$ such that
 	\begin{align*}
 		\bE\bb[ \sup_{0 \le t \le T} |X^{\Pi,\bfh}_t|^2\bb] \le K(1 + |x|^2).
 	\end{align*}
 Furthermore, the strong solution $X^{\Pi,\bfh}$ also solves the  SDE
 \begin{align}\label{eq:SDE-random_measure_Pi}
 	X^{\Pi,\bfh}_t = x &+ \int_{(0,t]\times[0,1]^d}  b_\bfh(s, X^{\Pi,\bfh}_{s-}, u)  M^\Pi_D (\od s, \od u)+ \sum_{l=1}^p \int_{(0,t]\times[0,1]^d}  a_\bfh^{(\cdot,l)}(s, X^{\Pi,\bfh}_{s-}, u) M^\Pi_{B^{(l)}}(\od s,\od u)  \notag\\ 
 	&+  	\int_{(0,t]\times \bR^q_0 \times[0,1]^d} \gamI_{\bfh}(s, X^{\Pi,\bfh}_{s-},  u, z)\tilde M^\Pi_{J}(\od s,\od z,\od u), \quad t \in [0, T],
 \end{align}
driven by the following random measures:
\begin{align*}
	&M_D^\Pi(\omega,\od t,\od u) := \sum_{i=1}^n  \1_{(t_{i-1}, t_i]}(t) \delta_{\xi^\Pi_{t_i}(\omega)}(\od u)\od t,\\
	&M_{B^{(l)}}^\Pi(\omega,t,A) := \bb(\int_0^{t}  \sum_{i=1}^n \1_{(t_{i-1},{t_i}]}(s) \1_A(\xi^\Pi_{t_i})\; \od B^{(l)}_s \bb)(\omega),\quad A\in  \cB([0,1]^d),\;t\in [0,T],\; l=1,\ldots, p,\\
	&M_J^\Pi(\omega,\od t,\od z,  \od u) := \sum_{i=1}^n  \sum_{t\in (t_{i-1},t_i]} \1_{\{\Delta L_t(\omega)\neq  0\}} \delta_{(t,\Delta L_t(\omega),\xi^\Pi_{t_i}(\omega))}(\od t,\od z, \od u),
\end{align*}
where $\delta_y$ denotes the Dirac distribution on the point $y$ and $L_t:=\int_{(0,t]\times \bR^q_0} z \tilde N(\od s,\od z)$; here,  $M^\Pi_{B^{(l)}}$ are orthogonal martingale measures with intensity measure $M_D^\Pi$, and $M^\Pi_J$ is an integer-valued random measure with predictable compensator measure
$$
\mu_J^\Pi(\omega,\od t,\od z,  \od u) := \sum_{i=1}^n  \1_{(t_{i-1}, t_i]}(t) \delta_{\xi^\Pi_{t_i}(\omega)}(\od u) \nu_t(\od z) \od t
$$
and  corresponding compensated jump measure $\tilde M^\Pi_J := M_J^\Pi-\mu_J^\Pi$.
  \end{theo}
 For integration with respect to compensated integer-valued random measures and to orthogonal martingale measures, we refer to \cite{JS03} and \cite{KM90}, respectively.
 \begin{proof}[Sketch of the proof.]
 The part on strong	existence, strong uniqueness  and the square integrability of the solution to \eqref{eq:SDE-Pi} is a variant of \cite[Theorem 1.19]{OS19}. It can be shown by standard arguments based on a Picard iteration and Gronwall's lemma, compare also the proof of \cref{prop:well-poseness-SDE-mm}. \cref{ass:Lip_grid} and the square integrability of  $X^{\Pi,\bfh}$ now imply that
 	\begin{eqnarray*}
 	\bE\bb[ \int_0^T  \|a_\bfh(s, X^{\Pi,\bfh}_{s-}, \xi^\Pi_s)\|^2_{\rmF} \od s\bb]\leq 2 TK_{\textnormal{Lip}} \, \bE\bb[ \sup_{0 \le t \le T} |X^{\Pi,\bfh}_t|^2\bb]+ 2\bE\bb[ \int_0^T  \|a_\bfh(s, 0, \xi^\Pi_s)\|^2_{\rmF} \od s\bb]<\infty.
 	\end{eqnarray*}
 Indeed, by Fubini's theorem the last term on the right-hand side becomes 
 $$
 2\int_0^T \nq \int_{[0, 1]^d}   \|a_\bfh(s, 0, u)\|_{\rmF}^2 \od u\leq 2\int_0^T G_0(s)\od s<\infty.
 $$
 In the same way, one can verify that
 $$
 	\bE\bb[ \int_0^T \nq \int_{\bR^q_0}  |\gamma_\bfh(s, X^{\Pi,\bfh}_{s-}, \xi^\Pi_s,z)|^2 \nu_s(\od z)\od s\bb]<\infty.
 $$
 Therefore, the integrals in the grid-sampling SDE can be manipulated by Lemmas 2.4--2.6 in \cite{BN24}, which shows that $X^{\Pi,\bfh}$ solves \eqref{eq:SDE-random_measure_Pi}.
 \end{proof}
 
\section{The grid-sampling limit SDE}\label{sec:limit}

The main theorem in \cite{BN24} proves vague convergence of the random measures $(M_D^\Pi, M_B^\Pi, M_J^\Pi)$ as the mesh-size of the partition $\Pi$ converges to zero. The random measure representation \eqref{eq:SDE-random_measure_Pi} of the grid-sampling SDE, then, motivates their definition of the grid-sampling limit SDE. For the convenience of the reader, we first recall the  limit theorem of \cite{BN24}.
\begin{theo}[Theorem 2.7 in \cite{BN24}]
		Let $(\Pi_n)_{n\in \bN}$ be a sequence of finite partitions of $[0,T]$ with $\lim_{n\rightarrow \infty} |\Pi_n|=0$.  For any $m\in \bN$, $R\in(0,\infty]$, and for any bounded measurable functions $f_{l}^{(k)}\colon [0,T] \times [0,1]^d\to \bR$ ($l=0,\ldots,p$; $k=1,\ldots,m$),   $f_{l}^{(k)} \colon [0,T]  \times \bR^q_0 \times [0,1]^d\to \bR$ ($l=p+1,p+2$; $k=1,\ldots,m$),
	consider the sequence of $\bR^m$-valued processes $\cX^n=(\cX^{n, (1)},\ldots, \cX^{n,(m)})$  defined via
	\begin{align*}
		\cX^{n,(k)}_t &= \int_{(0,t]\times[0,1]^d} f_{0}^{(k)}(s,u) M_D^{\Pi_n}(\od s,\od u)+ \sum_{l=1}^p \int_{(0,t]\times[0,1]^d} f_{l}^{(k)}(s,u) M_{B^{(l)}}^{\Pi_n}(\od s,\od u)  \\ 
		& \quad +  	\int_{(0,t]\times \{ 0<|z|\leq R\} \times[0,1]^d} f_{p+1}^{(k)}(s,z,u)|z| \tilde M_{J}^{\Pi_n}(\od s,\od z,\od u) \\ 
		& \quad +  	\int_{(0,t]\times \{ |z|> R\} \times[0,1]^d} f_{p+2}^{(k)}(s,z,u) M_{J}^{\Pi_n}(\od s,\od z,\od u), \quad t\in [0,T] ,\; k=1,\ldots, m.
	\end{align*}
	Then, $(\cX^n)_{n\in \bN}$ converges weakly to $\cX=(\cX^{(1)},\ldots, \cX^{(m)} )$ in the Skorokhod topology on the space of $\bR^m$-valued, c\`adl\`ag functions, where
	\begin{align*}
		\cX^{(k)}_t &= \int_{(0,t]\times[0,1]^d} f_{0}^{(k)}(s,u)\, \od s\,\od u+ \sum_{l=1}^p \int_{(0,t]\times[0,1]^d}f_{l}^{(k)}(s,u) M_{B^{(l)}}(\od s,\od u)  \\ 
		& \quad +  	\int_{(0,t]\times \{ 0<|z|\leq R\} \times[0,1]^d} f_{p+1}^{(k)}(s,z,u)|z| \tilde M_{J}(\od s,\od z,\od u) \\ 
		& \quad +  	\int_{(0,t]\times \{ |z|> R\} \times[0,1]^d} f_{p+2}^{(k)}(s,z,u) M_{J}(\od s,\od z,\od u), \quad t\in [0,T] ,\; k=1,\ldots, m.
	\end{align*}
Here, $M_{B^{(1)}},\ldots, M_{B^{(p)}}$ are independent white noise martingale  measures on $[0, T] \times \cB([0,1]^d)$ with intensity measure $\Leb_{[0,T]}\otimes \Leb^{\otimes d}_{[0,1]}$ and $M_J$ is an (inhomogeneous) Poisson random measure on $[0,T] \times \bR^q_0\times [0,1]^d$ (independent of the white noise measures) with predictable compensator 
$
\mu_J(\od t, \od z,\od u) := \nu_t(\od z)\od u \od t
$ and compensated measure $\tilde M_J : = M_J - \mu_J$.
\end{theo}
For more information on white noise measures, we refer to \cite{KM90}.

Replacing the random measures $(M_D^\Pi, M_B^\Pi, M_J^\Pi)$ by their limit measures in \eqref{eq:SDE-random_measure_Pi}, one arrives at the \emph{grid-sampling limit SDE} for a randomized policy $\bfh$, namely,
\begin{align}\label{eq:SDE-random_measure_limit}
	X^{\bfh}_t = x &+ \int_{0}^t \nq \int_{[0,1]^d}  b_\bfh(s, X^{\bfh}_{s-}, u) \od u \od s+ \sum_{l=1}^p \int_{(0,t]\times[0,1]^d}  a_\bfh^{(\cdot,l)}(s, X^{\bfh}_{s-}, u) M_{B^{(l)}}(\od s,\od u)  \notag\\ 
	&+  	\int_{(0,t]\times {\bR^q_0} \times[0,1]^d}  \gamI_{\bfh}(s, X^{\bfh}_{s-},  u, z)\tilde M_{J}(\od s,\od z,\od u).
\end{align}

The following theorem provides the wellposedness of the grid-sampling limit SDE under slightly weaker assumptions than the ones imposed in \cref{ass:Lip_grid}. 
\begin{theo}\label{thm:well-posedness-SDE}
	Suppose that $b_{\bfh}$, $a_{\bfh}$, and $\gamma_{\bfh}$ are Borel measurable and 
	 that there exist constants $K_{b, \bfh}, K_{a, \bfh}, K_{\gamma, \bfh} \ge 0$ independent of $s$ such that the following conditions hold for any $s \in [0, T]$ and $x_1, x_2 \in \bR^m$: The Lipschitz property for the drift part
	\begin{align*}
		\int_{[0, 1]^d} |b_\bfh(s, x_1, u) - b_\bfh(s, x_2, u)| \od u \le K_{b, \bfh} |x_1 - x_2|,
	\end{align*}
	and for the martingale parts
	\begin{align*}
		& \int_{[0, 1]^d} \|a_\bfh(s, x_1, u) - a_\bfh(s, x_2,u)\|_{\rmF}^2 \, \od u \le K^2_{a, \bfh}|x_1 - x_2|^2, \\
		& \int_{\bR^q_0 \times [0, 1]^d} |\gamma_\bfh(s, x_1, u, z) - \gamma_\bfh(s, x_2, u, z)|^2 \, \nu_s(\od z) \od u \le K^2_{\gamma, \bfh}|x_1 - x_2|^2,
	\end{align*}
	and the function $G_0$ introduced in \cref{ass:Lip_grid}
	takes finite values for all $s \in [0, T]$ and $G_0 \in \bfL^1([0, T], \Leb_{[0, T]})$. 	Then, the grid-sampling limit SDE \eqref{eq:SDE-random_measure_limit} for policy $\bfh$ has a unique (up to an indistinguishability) strong solution $X^{\bfh}$ with
	\begin{align*}
		\bE\bb[ \sup_{0 \le t \le T} |X^{\bfh}_t|^2\bb] \le K(1 + |x|^2)
	\end{align*}
	for some constant $K\ge 0$ independent of the initial condition $x\in \bR^m$. Moreover, the law of $X^\bfh$ solves the martingale problem for the operator $\cL_\bfh$ defined for $f \in C^2_c(\bR^m)$ by
	\begin{align*}
		(\cL_{\bfh} f)(s, x) & := \int_{[0, 1]^d} \bb[\sum_{i=1}^m b^{(i)}_{\bfh}(s, x, u) \frac{\pd f}{\pd x_i}(x) +  \frac{1}{2}\sum_{i, j = 1}^m A_\bfh^{(i, j)}(s, x, u) \frac{\pd^2 f}{\pd x_i \pd x_j}(x) \\
		& \hspace{1.5cm} + \int_{\bR^q_0} \bb(f(x + \gamma_\bfh(s, x, u, z)) - f(x) - \sum_{i=1}^m  \gamma_\bfh^{(i)}(s, x, u, z) \frac{\pd f}{\pd x_i}(x) \bb) \nu_s(\od z) \bb] \od u,
	\end{align*}
	with initial distribution $\delta_x$. Here, $C^2_c(\bR^m)$ stands for the space of twice-continuously differentiable functions with compact support in $\bR^m$, and $A_\bfh : = a_\bfh a_\bfh^\tran$.
\end{theo}

 \begin{proof}[Proof of  \cref{thm:well-posedness-SDE}]
  We aim to apply \cref{prop:well-poseness-SDE-mm}. To do that, we employ the notations there and choose $E = \bR^q \times [0, 1]^d$ which is equipped with the Euclidean norm. Let $\ell = p +1$. Define the orthogonal martingale measures on $[0, T] \times \cB(\bR^q \times [0, 1]^d)$ by
 		\begin{align*}
 			M^{(j)}(\od s, \od z, \od u) := \begin{cases}
 				\delta_0(\od z) M_{B^{(j)}}(\od s, \od u) & \trm{if } j = 1, \ldots, p,\\
 				\1_{\{z \neq 0\}} |z| \tilde M_J(\od s, \od z, \od u) & \trm{if } j = p+1,
 			\end{cases}
 		\end{align*}
 		with the intensity
 		\begin{align*}
 			\mu^{(j)}(\od s, \od z, \od u) := \begin{cases}
 				\delta_0(\od z)  \od u \od s & \trm{if } j = 1, \ldots, p,\\
 				\1_{\{z \neq 0\}} |z|^2 \nu_s(\od z) \od u \od s & \trm{if } j = p+1.
 			\end{cases}
 		\end{align*}
 		Now, the $\bR^m$-valued function $\beta$ and $\bR^{m \times (p+1)}$-valued function $\alpha$ are defined by 
 		\begin{align*}
 			\beta(s,x) &:= \int_{[0, 1]^d} b_\bfh(s, x, u) \od u,\\
 			\alpha^{(i, j)}(s, x, z, u) &:=  \begin{cases}
 				\1_{\{z = 0\}} a_\bfh^{(i, j)}(s, x, u) & \trm{if } 1 \le i \le m, 1 \le j \le p,\\
 				\1_{\{z \neq 0\}} |z|^{-1} \gamma_\bfh^{(i)}(s, x, u, z) & \trm{if } 1 \le i \le m, j = p+1.
 			\end{cases}
 		\end{align*}
 		Then, the SDE \eqref{eq:SDE-random_measure_limit} becomes exactly  \eqref{app:SDE-modified} and all conditions in \cref{prop:well-poseness-SDE-mm} are fulfilled, and hence, the unique existence of the strong solution with the moment estimate follows.
 		
 		\smallskip
 		
 		Regarding the martingale problem, we first notice that $\cL_\bfh f$ is finitely defined on $[0, T] \times \bR^m$ due to the conditions imposed on the coefficients and the boundedness of the partial derivatives of $f$.
 		Note that the continuous martingale part of $X^{\bfh}$ has predictable quadratic variation equal to 
 		$$\int_0^\cdot  \sum_{i, j =1}^m \bb(\int_{[0, 1]^d} A^{(i, j)}_{\bfh}(s, X^\bfh_{s-}, u) \od u \bb) \od s $$ 
 		by combining Proposition I-6(2) in \cite{KM90} with taking the independence of the white noise measures into account.
 			We may, thus, apply It\^o's formula for $X^\bfh$ and $f \in C^2_c(\bR^m)$ to get, a.s., 
 		\begin{align*}
 			&f(X^\bfh_t) - f(x) \\
 			& = \int_0^t \sum_{i=1}^m \frac{\pd f}{\pd x_i}(X^\bfh_{s-}) \bb(\int_{[0, 1]^d} b^{(i)}_\bfh(s, X^{\bfh}_{s-}, u) \od u\bb) \od s\\
 			& \quad + \frac{1}{2} \int_0^t \sum_{i, j =1}^m \frac{\pd^2 f}{\pd x_i \pd x_j}(X^\bfh_{s-}) \bb(\int_{[0, 1]^d} A^{(i, j)}_{\bfh}(s, X^\bfh_{s-}, u) \od u \bb) \od s\\
 			& \quad + \int_0^t \nq \int_{\bR^q_0 \times [0, 1]^d} \bb(f(X^\bfh_{s-} + \gamma_\bfh(s, X^\bfh_{s-}, u, z)) - f(X^\bfh_{s-}) - \sum_{i=1}^m \frac{\pd f}{\pd x_i}(X^\bfh_{s-}) \gamma^{(i)}_\bfh(s, X^\bfh_{s-}, u, z)\bb) \nu_s(\od z) \od u \od s \\
 			& \quad + \trm{local martingale terms}\\
 			& = \int_0^t (\cL_\bfh f)(s, X^\bfh_{s-}) \od s + \trm{local martingale terms}.
 		\end{align*}
 		Since $f$ has compact support, the local martingale terms become a proper bounded martingale. Hence, the law of $X^\bfh$ solves the said martingale problem above.
 	\end{proof}

 	\begin{rema}
 		In \cite{BN24} the grid-sampling limit SDE is introduced in the more general form
 		\begin{align*}
 		X^{\bfh}_t = x &+ \int_{0}^t \nq \int_{[0,1]^d}  b_\bfh(s, X^{\bfh}_{s-}, u) \od u \od s+ \sum_{l=1}^p \int_{(0,t]\times[0,1]^d}  a_\bfh^{(\cdot,l)}(s, X^{\bfh}_{s-}, u) M_{B^{(l)}}(\od s,\od u)  \notag\\
 			&+  	\int_{(0,t]\times \{ 0<|z|\le \mathfrak{r}\} \times[0,1]^d} \gamI_{\bfh}(s, X^{\bfh}_{s-},  u, z)\tilde M_{J}(\od s,\od z,\od u) \notag \\ 
 			& +  	\int_{(0,t]\times \{ |z|> \mathfrak{r}\} \times[0,1]^d} \gamI_{\bfh}(s, X^{\bfh}_{s-},  u, z) M_{J}(\od s,\od z,\od u)
 		\end{align*}
 		for some $\mathfrak{r}\in [0,\infty]$. In this case the integrability condition \eqref{assumption:Levy-measure-r} can be replaced by 
 		$$
 		 \int_0^T \nq \int_{\bR^q_0} ( |z|^2 \1_{\{0 < |z| \le \mathfrak{r}\}}  + \1_{\{|z| > \mathfrak{r}\}}) \nu_t(\od z)  \od t < \infty,
 		$$ 
 		which is weaker than \eqref{assumption:Levy-measure-r}, if $\mathfrak{r}\in (0,\infty)$. In the presence of big jumps that may fail to be square integrable, wellposedness of the grid-sampling SDE can be established by adapting the interlacing technique for the finite activity jump part, see, e.g., the proof of \cite[Theorem IV.9.1]{IW89}. In this case, the solution cannot be expected to be square integrable anymore. Therefore, integration with respect to the orthogonal martingale measures must be localized, see, e.g., Chapter 13 in \cite{KD01}. 		
 	\end{rema}

\appendix

\section{SDEs driven by martingale measures}

Let $\{M^{(1)}, \ldots, M^{(\ell)}\}$ be a collection of  (c\`adl\`ag) $(\bF, \bP)$-martingale measures on $[0, T] \times \cB(E)$, where $E$ is a complete and separable metric space equipped with the Borel $\sigma$-field $\cB(E)$, see \cite{KM90, Wa86}.
Assume that each $M^{(j)}$ is an orthogonal martingale measure with \textit{deterministic} intensity measure $\mu^{(j)}$ which satisfies 
$$\mu^{(j)}(\od s, \od e) = \mu^{(j)}_s(\od e) \od s$$ for some transition kernel $(\mu^{(j)}_s)_{s \in [0, T]}$, $j = 1, \ldots, \ell$.
We denote $M = (M^{(1)}, \ldots, M^{(\ell)})^\tran$.

\begin{prop}\label{prop:well-poseness-SDE-mm} Let $\beta \colon [0, T] \times \bR^m \to \bR^m$ and $\alpha \colon [0, T] \times \bR^m \times E \to \bR^{m \times \ell}$ be  measurable. Assume that there exist  constants $K_\beta, K_\alpha \ge 0$ not depending on $s$ such that, for all $s\in [0, T]$, $y_1, y_2 \in \bR^m$,
	\begin{equation}\label{app:Lipschitz-SDE}
		\begin{aligned}
			&|\beta(s, y_1) - \beta(s, y_2)| \le K_\beta |y_1 - y_2|, \\
			&4 \ell \sum_{i=1}^m \sum_{j=1}^\ell \int_E |\alpha^{(i, j)}(s, y_1, e) - \alpha^{(i, j)}(s, y_2, e)|^2 \,  \mu^{(j)}_s(\od e) \le K_\alpha^2 |y_1 - y_2|^2,
		\end{aligned}
	\end{equation}
	and that
	\begin{align}\label{app:growth-condition}
		K^2_0: =	T \int_0^T |\beta(s, 0)|^2 \od s + 4\ell \sum_{i=1}^m \sum_{j=1}^\ell \int_0^T \nq \int_E |\alpha^{(i, j)}(s, 0, e)|^2\, \mu^{(j)}_s(\od e) \od s < \infty. 
	\end{align}
	Then, for any $\cF_0$-measurable $Y_0 \in \bfL^2(\bP)$, the following SDE 	\begin{align}\label{app:SDE-modified}
		Y_t = Y_0 + \int_0^t \beta(s, Y_{s-}) \od s + \int_{(0, t] \times E} \alpha(s, Y_{s-}, e) M(\od s, \od e), \quad t \in [0, T],
	\end{align}
	has a unique (up to an indistinguishability) strong solution $Y$ with
	\begin{align*}
		\bE\bb[\sup_{0 \le t \le T} |Y_t|^2\bb] \le K (1+ \bE[|Y_0|^2])
	\end{align*}
	for some constant $K \ge 0 $ depending only on $K_\beta, K_\alpha, K_0$ and $T$.
\end{prop}

Related results are stated, e.g., in \cite[Proposition IV-1]{KM90} for the no-jump case and in \cite[Remark~4]{BCP20}.

\begin{proof}
	
	\noindent \textbf{\textit{Existence.}} 
	We use the Picard iteration argument. Let $Y^0 = (Y^0_{t})_{t \in [0, T]}$ with $Y^0_t: = Y_0$ for all $t \in [0, T]$, and inductively define the sequence of process $(Y^n)_{n \in \bN}$ via
	\begin{align*}
		Y^n_t : = Y_0 + \int_0^t  \beta(s, Y^{n-1}_{s-}) \od s + \int_{(0, t] \times E} \alpha(s, Y^{n-1}_{s-}, e) M(\od s, \od e), \quad t \in [0, T].
	\end{align*}
	We also define
	\begin{align*}
		\Theta_t : = \int_0^t  \beta(s, 0) \od s + \int_{(0, t] \times E} \alpha(s, 0, e) M(\od s, \od e), \quad t \in [0, T].
	\end{align*}
	Combining Doob's maximal inequality, the inequality $(x_1 + \cdots + x_\ell)^2 \le \ell (x_1^2 + \cdots + x_\ell^2)$, It\^o's isometry with using \eqref{app:growth-condition} we infer that $\Theta$ is an adapted and $\bR^m$-valued c\`adl\`ag process with
	\begin{align*}
		\bE\bb[ \sup_{0 \le t \le T} |\Theta_t|^2\bb] \le 2\bb[ T \int_0^T |\beta(s, 0)|^2 \od s + 4 \ell \sum_{i=1}^m \sum_{j=1}^\ell \int_0^T \nq \int_E |\alpha^{(i, j)}(s, 0, e)|^2\, \mu^{(j)}_s(\od e) \od s\bb] = 2 K^2_0.
	\end{align*}
	By \eqref{app:Lipschitz-SDE} and the square integrability of $Y^0$, together with Fubini's theorem, we get 
	\begin{align}\label{app:eq:constant-0}
		\bE\bb[ \sup_{0 \le t \le T} |Y^1_t - Y^0_t|^2\bb] & \le 2\bE\bb[ \sup_{0 \le t \le T} |Y^1_t - Y^0_t - \Theta_t|^2\bb]  + 2\bE\bb[ \sup_{0 \le t \le T} |\Theta_t|^2\bb]  \notag \\
		& \le 4  T^2 K_\beta^2 \, \bE[|Y_0|^2] + 4 T K^2_\alpha \, \bE[|Y_0|^2] + 4 K^2_0 \notag \\
		& \le K^2_{\eqref{app:eq:constant-0}} ( 1+ \bE[|Y_0|^2]) \tag{1i}
	\end{align}
	for $K^2_{\eqref{app:eq:constant-0}} := 4\max\{K^2_0,   T^2 K_\beta^2 +  T K^2_\alpha\}$. 
	We then deduce by induction using the same arguments as above that $Y^n$ is well-defined and square integrable for all $n \in \bN$. For any $n \ge 1$ and $t \in [0, T]$,
	\begin{align} \label{app:eq:estimate-Y}
		\frac{1}{2}\bE \bb[\sup_{0 \le s \le t} |Y^{n+1}_s - Y^{n}_s|^2\bb] & \le \bE\bb[\sup_{0 \le s \le t}  \bb|\int_0^s [\beta(r, Y^n_{r-}) - \beta(r, Y^{n-1}_{r-})] \od r\bb|^2\bb] \notag \\
		& \quad +  \bE\bb[\sup_{0 \le s \le t} \bb| \int_{(0, s] \times E} [\alpha(r, Y^n_{r-}, e) - \alpha(r, Y^{n-1}_{r-}, e)] M(\od r, \od e)\bb|^2 \bb] \notag \\
		& \le t \bE\bb[\int_0^t |\beta(s, Y^n_{s}) - \beta(s, Y^{n-1}_{s})|^2\od s\bb] \notag \\
		& \quad + 4 \ell \sum_{i=1}^m \sum_{j=1}^\ell \bE\bb[\int_0^t \nq \int_E |\alpha^{(i, j)}(s, Y^n_{s}, e) - \alpha^{(i, j)}(s, Y^{n-1}_{s}, e)|^2 \,  \mu^{(j)}_s(\od e) \od s\bb] \notag \\
		& \le (t K_\beta^2  +  K_\alpha^2) \,  \bE\bb[\int_0^t |Y^n_s - Y^{n-1}_s|^2 \od s\bb],
	\end{align}
	where we use Doob's maximal inequality in the second inequality and use Fubini's theorem in the third inequality.
	Then, for 
	\begin{align}\label{app:constant}
		K^2_{\eqref{app:constant}} : = 2 (T K_\beta^2  +  K_\alpha^2), \tag{2i}
	\end{align}
	one has
	\begin{align*}
		\bE \bb[\sup_{0 \le s \le t} |Y^{n+1}_s - Y^{n}_s|^2\bb] \le K^2_{\eqref{app:constant}} \int_0^t \bE \bb[\sup_{0 \le r \le s} |Y^{n}_r - Y^{n-1}_r|^2\bb] \od s, \quad \forall t \in [0, T], n \ge 1.
	\end{align*}
	Iterating the estimate above, we get for any $n \in \bN$,
	\begin{align}\label{app:eq:estimate-sup-1}
		\bE \bb[\sup_{0 \le s \le T} |Y^{n+1}_s - Y^{n}_s|^2\bb]  \le K^{2n}_{\eqref{app:constant}} \frac{T^n}{n!} \bE \bb[\sup_{0 \le r \le T} |Y^{1}_r - Y^{0}_r|^2\bb].
	\end{align}
	Combining Markov's inequality with \eqref{app:eq:estimate-sup-1}  we get
	\begin{align*}
		\sum_{n = 0}^\infty \bP\bb(\bb\{\sup_{0 \le s \le T} |Y^{n+1}_s - Y^{n}_s| \ge \frac{1}{2^n}  \bb\} \bb) \le \bE \bb[\sup_{0 \le r \le T} |Y^{1}_r - Y^{0}_r|^2\bb]  \sum_{n=0}^\infty \frac{\b(4T K^2_{\eqref{app:constant}}\b)^n}{n!} < \infty.
	\end{align*}
	By the Borel--Cantelli lemma, there is an event $\Omega_0$ with probability one such that for any $\omega \in \Omega_0$, there exists $n_\omega \in \bN$ such that
	\begin{align*}
		\sup_{0 \le s \le T} |Y^{n+1}_s(\omega) - Y^n_s(\omega)| < \frac{1}{2^n}, \quad \forall n \ge n_\omega.
	\end{align*}
	We deduce that $Y^n(\omega)$ converges uniformly on $[0, T]$ for $\omega \in \Omega_0$. We define, for all $t \in [0, T]$,
	\begin{align*}
		Y_t(\omega) : = \begin{cases}
			\lim_{n \to \infty} Y^n_t(\omega)  & \trm{if } \omega \in \Omega_0\\
			0 & \trm{if } \omega \notin \Omega_0.
		\end{cases}
	\end{align*}
	By the uniform convergence and the completeness of the underlying filtration, $Y$ has c\`adl\`ag paths and is adapted.
	
	Now, by the triangle inequality, it follows from \eqref{app:eq:estimate-sup-1} that 
	\begin{align*}
		\bb\|\sup_{0 \le s \le T} |Y_s - Y^{0}_s|\bb\|_{\bfL^2(\bP)}  & = \bb\|\lim_{n \to \infty} \sup_{0 \le s \le T} |Y_s^{n} - Y^{0}_s|\bb\|_{\bfL^2(\bP)} \\
		& \le \sum_{j=0}^\infty K^j_{\eqref{app:constant}} \sqrt{\frac{T^j}{j!}} \, \bb\|\sup_{0 \le r \le T} |Y^{1}_r - Y^{0}_r| \bb\|_{\bfL^2(\bP)} \\
		& \le \sqrt{2 \e^{2K^2_{\eqref{app:constant}} T}} K_{\eqref{app:eq:constant-0}}  \sqrt{ 1+ \bE[|Y_0|^2]},
	\end{align*}
	which then yields
	\begin{align}\label{app:solution-sup-L2}
		\bE\bb[\sup_{0 \le s \le T} |Y_s|^2 \bb]
		& \le 2 \bE[|Y_0|^2] + 4 \e^{2 K^2_{\eqref{app:constant}} T} K^2_{\eqref{app:eq:constant-0}} ( 1+ \bE[|Y_0|^2])  \le K^2_{\eqref{app:eq:constant-3}} ( 1+ \bE[|Y_0|^2]),
	\end{align}
	where
	\begin{align}\label{app:eq:constant-3}
		K^2_{\eqref{app:eq:constant-3}} : = 2 + 4 \e^{2 K^2_{\eqref{app:constant}} T} K^2_{\eqref{app:eq:constant-0}}. \tag{3i}
	\end{align}
	By \eqref{app:Lipschitz-SDE}, \eqref{app:growth-condition} and \eqref{app:solution-sup-L2}, the following process $Z$ is well-defined in $\bfL^2(\bP)$,
	\begin{align*}
		Z_t : = Y_0 + \int_0^t  \beta(s, Y_{s-}) \od s + \int_{(0, t] \times E} \alpha(s, Y_{s-}, e) M(\od s, \od e), \quad t \in [0, T].
	\end{align*}
	We now show that $Z = Y$. Indeed, proceeding as in  \eqref{app:eq:estimate-Y} with $Z$ in place of $Y^n$, we get
	\begin{align*}
		\bb\|\sup_{0 \le s \le T} |Z_s - Y^{n+1}_s|\bb\|_{\bfL^2(\bP)} 
		& \le \sqrt{T} K_{\eqref{app:constant}}  \bb\|\sup_{0 \le s \le T} |Y_s - Y^{n}_s| \bb\|_{\bfL^2(\bP)}\\
		& \le \sqrt{T} K_{\eqref{app:constant}} \sum_{j = n}^\infty K^j_{\eqref{app:constant}} \sqrt{\frac{T^j}{j!}} \, \bb\|\sup_{0 \le r \le T} |Y^{1}_r - Y^{0}_r| \bb\|_{\bfL^2(\bP)} \\
		& \xrightarrow{n \to \infty} 0,
	\end{align*}
	which completes the proof for the existence.
	
	\smallskip
	
	\noindent \textbf{\textit{Uniqueness.}} Assume that $Y$ and $\tilde Y$ solve \eqref{app:SDE-modified}. We follow the same estimates as for \eqref{app:eq:estimate-Y} to get
	\begin{align*}
		\bE\bb[ \sup_{0 \le s \le t} |Y_s - \tilde Y_s|^2 \bb] \le K^2_{\eqref{app:constant}} \int_0^t \bE\bb[ \sup_{0 \le r \le s} |Y_r - \tilde Y_r|^2  \bb] \od s, \quad \forall t \in [0, T].
	\end{align*}
	Note that $\bE\b[\sup_{0 \le s \le T} |Y_s|^2\b] < \infty$ as well as for $\tilde Y$, which verifies $\bE\b[\sup_{0 \le s \le T} |Y_s - \tilde Y_s|^2\b] < \infty$. Applying Gronwall's lemma yields
	\begin{align*}
		\bE\bb[ \sup_{0 \le s \le T} |Y_s - \tilde Y_s|^2 \bb] = 0,
	\end{align*}
	which then shows $Y = \tilde Y$.
\end{proof}


\bibliographystyle{amsplain}

\end{document}